\documentclass[12pt,pdftex]{article}
\usepackage{amsmath, amssymb}
\usepackage[a4paper]{geometry}
\usepackage{graphicx}
\usepackage{dsfont}
\usepackage{xspace}

\newcommand{\wlo}{w.\,l.\,o.\,g.\xspace}
\newcommand{\ie}{i.\,e.\xspace}
\newcommand{\eg}{e.\,g.\xspace}
\newcommand{\oneoneea}{(1+1)\nobreakdash-EA\xspace}

\newcommand{\card}[1]{|#1|}

\newcommand{\N}{\mathds{N}}
\newcommand{\R}{\mathds{R}}

\renewcommand{\epsilon}{\varepsilon}
\DeclareMathOperator{\Prob}{Prob}

\newcommand{\xmin}{x_{\min}}
\newcommand{\ftot}{f_{\mathrm{tot}}}

\newtheorem{lemma}{Lemma}

\newtheorem{theorem}{Theorem}

\newenvironment{proof}%
{\begin{trivlist}\item\textbf{Proof:}}%
{\hspace*{\fill}$\Box$\end{trivlist}}

{\begin{trivlist}\item\textbf{Proof of #1:}}%
{\hspace*{\fill}$\Box$\end{trivlist}}

\title{Erratum: Simplified Drift Analysis for Proving  
Lower Bounds in Evolutionary Computation}

\author{Pietro S. Oliveto\\
\small School of Computer Science,\\
\small University of Birmingham,\\
\small  Birmingham, UK
\and Carsten Witt\\
\small DTU Informatics,\\
\small Technical University of Denmark,\\
\small  Kgs.~Lyngby, Denmark
}

\begin{document}

\maketitle

\begin{abstract}
\noindent
This erratum points out an error in the simplified drift theorem (SDT) 
\cite{OlivetoWittPPSN2008, OlivetoWittAlgorithmica2011}. It is also shown 
that a minor modification of one of its conditions is sufficient to 
establish a valid result. In many respects, the new theorem is 
more general than before. 
We no longer assume a Markov process nor a finite search space. 
Furthermore, the proof of the theorem is more compact than the previous ones. 
Finally, previous applications of the SDT are revisited. 
It turns out that all of these either meet the modified condition directly 
or by means of few additional arguments.
\end{abstract}

\section{Introduction}
The so-called simplified drift theorem, first presented in \cite{OlivetoWittPPSN2008},  
deals with stochastic processes that drift away from a target, \ie, processes that 
in expectation increase the distance from the target in a step. For example, 
if $X_t>0$ is the state of the process at time~$t$ and 
the aim is to reach the target state~$0$, then the SDT deals with 
the case 
$E(X_{t+1}-X_{t})\ge \epsilon$. (Sometimes this is called 
negative drift, if perspectives are switched and the aim is to reach 
a maximal state.)

The aim is to show that the process takes a long time to reach the target. 
Intuitively, a drift away from the target is not enough for this. It might well be 
that the expected change is positive, but a direct jump to the optimum occurs with 
a fair probability. Therefore, the simplified drift theorem contains 
a second condition in addition to the drift condition, namely an 
exponential decay of the probability of jumping towards the target, 
more precisely a condition of the kind  
$\Prob(X_{t+1}-X_{t} \le -j ) \le 2^{-j}$ for all natural $j>0$. 
See \cite{OlivetoWittPPSN2008} and 
the extended journal version \cite{OlivetoWittAlgorithmica2011} for a 
precise formulation.

The SDT has found many applications and had a large impact on the way lower 
bounds on the optimization time of randomized search heuristics are proved. 
Unfortunately, we recently discovered an error in its proof. 
In fact all existing variants of the SDT presented before, in any case 
the ones in \cite{OlivetoWittPPSN2008, OlivetoWittAlgorithmica2011}, seem to be
wrong 
unless the second condition is strengthened to something being 
essentially like
$\Prob(\lvert X_{t+1}-X_{t}\vert \ge j ) \le 2^{-j}$, \ie, 
we should also have an exponential decay for the
jumps \emph{towards} the target (note the absolute value).

This erratum is structured as follows. We will present a counterexample to the previous SDT 
in Section~\ref{sec:counter}. In Section~\ref{sec:corrected}, we will present 
and prove a corrected version of the SDT. In Section~\ref{sec:applications}, we 
show that seemingly all applications of the original SDT also satisfy the stronger 
conditions, either immediately or after a few additional arguments.

\section{An Example Where the Original Theorem is Wrong}
\label{sec:counter}

Consider the SDT as presented in \cite{OlivetoWittAlgorithmica2011} (which 
essentially is the new 
Theorem~\ref{theo:modifieddrift} in Section~\ref{sec:corrected} but with 
the weaker second condition 
$\Prob(X_{t+1}-X_t\le -j \mid a< X_t) \le  \frac{r(\ell)}{(1+\delta)^{j}}$).

Let us look into the following Markov process on some  state
space being exponentially large in~$n$, 
say $S=\{-3e^{n},...,3e^n\}$ (the precise size does not matter, 
but the original drift theorem demands a finite search space). 
We set $a:=0$ as target and $b:=n$ as starting point, \ie, 
$X_0 = b = n$.

The stochastic behavior of the process is as follows: Conditioned on 
$X_t\in \mathopen{]}a,b\mathclose{]}$, we have $X_{t+1}=X_t + 2 e^n$ with probability $e^{-n}$,
and $X_{t+1} = X_t - 1$ with the remaining probability $1-e^{-n}$. 
Note that the 
steps towards the target only have size~$1$. 
The behavior in the case $\{X_t\le a \vee X_t> b\}$ is not important, say 
$X_{t+1}=X_t$ then (process stops).

We get $E(X_{t+1}-X_t \mid X_t\in \mathopen{]}a,b\mathclose{[}) = e^{-n} \cdot 2e^n + (1-e^{-n}) \cdot (-1) \ge 1$, hence 
there is constant drift away from the target towards $b$ within the drift interval.

However, it is very likely (probability at least $1-ne^{-n}$) that
starting from $b$, the process takes $n$ decreasing steps in a row and reaches $a$.

The ``proof'' of the SDT presented in \cite{OlivetoWittPPSN2008, OlivetoWittAlgorithmica2011} 
erroneously estimates a double sum appearing in the moment-generating 
function of the drift from above. More precisely, 
all single terms are uniformly bounded without paying attention to their sign.

\section{The Corrected Version}
\label{sec:corrected}

Our aim is to present a corrected simplified drift theorem, which 
as before deals with 
drift away from the target and holds in both 
discrete and continuous search spaces. To this end, the 
following lemma will be useful.

\begin{lemma}
\label{lem:general-bound-expec}
Let $X$ be a random variable 
with minimum $x_{\min}$. Moreover, let $f\colon \R\to \R$ be a 
non-decreasing function and suppose that the expectation 
$E(f(X))$ exists. Then
\[
E(f(X)) \le \sum_{i=0}^\infty f(x_{\min}+ i+1) \cdot \Prob(X\ge x_{\min} + i).
\]
\end{lemma}

\begin{proof}
We denote by $p$ the probability measure from 
the probability space $(\Omega,\Sigma,p)$ underlying~$X$. Then the expectation is given 
by a Lebesgue integral, more precisely
\[
E(f(X))=\int_{\Omega} f(X(\omega)) \,p(\mathrm{d}\omega)
\]
Since $f$ is non-decreasing and $X\ge \xmin$, partial integration yields
\begin{align*}
E(f(X)) & \le \sum_{i=0}^\infty 
f(\xmin+i+1) \int_{\Omega\cap X^{-1}([\xmin+i,\xmin+i+1])} p(\mathrm{d}\omega) \\
& \le 
\sum_{i=0}^\infty 
f(\xmin + i+1) \Prob(X\ge \xmin+  i).
\end{align*}
\end{proof}

\label{sec:newdrift}

We will use Hajek's following drift 
theorem to prove our result.  
In contrast to \cite{OlivetoWittAlgorithmica2011}, our presentation  
of Hajek's drift theorem does not make unnecessary assumptions 
such as non-negativity of the random variables or Markovian processes. 
As we are dealing with a stochastic process, we 
implicitly assume that the random variables~$X_t$, $t\ge 0$,  
are adapted to the natural filtration~$\mathcal{F}_t=(X_0,\dots,X_t)$, though.

We do no longer formulate the theorem using a ``potential function'' $g$ 
mapping from some state space to the reals either. Instead,  
we \wlo\ assume the random variables~$X_t$ as already obtained by the 
mapping.

\begin{theorem}[Hajek \cite{Hajek1982}]
  \label{theo:orig-drift}
  Let $X_t$, $t\ge 0$, be real-valued random variables describing a
	stochastic process over some state space. Pick two
  real numbers $a(\ell)$ and $b(\ell)$ depending on a parameter~$\ell$
  such that $a(\ell)<b(\ell)$ holds. Let $T(\ell)$ be the random
  variable denoting the earliest point in time $t\ge 0$ such that
  $X_t\le a(\ell)$ holds.  If there are $\lambda(\ell)>0$ and
  $p(\ell)>0$ such that the condition
  \begin{equation}
  \label{eq:maindriftcondition}
  \tag{$\ast$}
  E\bigl(e^{-\lambda(\ell)\cdot (X_{t+1}-X_t)}
     \mid \mathcal{F}_t \wedge\; a(\ell)<  X_t < b(\ell)\bigr)
  \;\,\le\,\; 1-\frac{1}{p(\ell)}  
  \end{equation}
  holds for all $t\ge 0$ then for all time bounds $L(\ell)\ge 0$
  \[
  \Prob\bigl(T(\ell)\le L(\ell) \mid X_0\ge b(\ell)\bigr) 
  \;\le\; e^{-\lambda(\ell)\cdot
    (b(\ell)-a(\ell))}\cdot L(\ell)\cdot D(\ell)\cdot p(\ell),
  \]
  where $D(\ell)=\max\bigl\{1,E\bigl(e^{-\lambda(\ell)\cdot (X_{t+1}-b(\ell))}\mid \mathcal{F}_t \;\wedge\; X_t\ge b(\ell)\bigr)\bigr\}$.
\end{theorem}

We now present the corrected simplified drift theorem. As discussed above, 
it combines a drift away from the target with a condition 
on exponentially decaying probabilities for large jumps 
\emph{both towards and away from the target}. Nevertheless, the 
theorem has become more general in other respects. More precisely, 
we do no longer assume a Markov process or a finite search space. 
At the same time, the proof is more compact than before.

\begin{theorem}[Simplified Drift Theorem] \label{theo:modifieddrift}
  Let $X_t$, $t\ge 0$, be real-valued random variables describing a
	stochastic process over some state space. 
  Suppose there exist an interval $[a,b]\subseteq \R$, two
  constants $\delta,\epsilon>0$ and, possibly depending on
  $\ell:=b-a$, a function $r(\ell)$ satisfying  
  $1\le r(\ell) = o(\ell/\!\log(\ell))$ such that  
  for all $t\ge 0$ the following two conditions hold:
  \begin{enumerate}
  \addtolength{\itemsep}{0.5ex}
  \item $E(X_{t+1}-X_{t}\mid \mathcal{F}_t \;\wedge\; a< X_t <b)  \;\ge\; \epsilon$, 
  \item $\Prob(\lvert X_{t+1}-X_t\rvert\ge j \mid \mathcal{F}_t \;\wedge\; a< X_t) \;\le\;  \frac{r(\ell)}{(1+\delta)^{j}}$ for  $j\in \N_0$.
  \end{enumerate}
  Then there is a constant~$c^*>0$ such that for $T^*:={\min\{t\ge
  0\colon X_t\le a \mid \mathcal{F}_t \wedge X_0\ge b\}}$ it holds $\Prob(T^*\le
  2^{c^*\ell/r(\ell)}) = 2^{-\Omega(\ell/r(\ell))}$.
\end{theorem}

\begin{proof}
 We will apply Theorem~\ref{theo:orig-drift} for suitable choices of
  its variables, some of which might depend on the parameter
  $\ell=b-a$ denoting the length of the interval~$[a,b]$.
  The following argumentation is
  also inspired by Hajek's work \cite{Hajek1982}.
	
  Fix $t\ge 0$. For notational convenience, 
	we let $\Delta:=(X_{t+1}-X_t\mid  \mathcal{F}_t \;\wedge\; a< X_t <b)$ and omit to state 
	the filtration $\mathcal{F}_t$ hereinafter. 
	To prove Condition~\eqref{eq:maindriftcondition}, it is sufficient
  to identify values $\lambda:=\lambda(\ell)>0$ and $p(\ell)>0$ such
  that
  \[
  E(e^{-\lambda \Delta}) 
  \;\le\; 1-\frac{1}{p(\ell)}.
  \]
  Using the series expansion of the exponential function, we get
  \begin{align*}
    E(e^{-\lambda \Delta}) \;=\; 
		1 - \lambda E(\Delta) + \lambda^2 \sum_{k=2}^\infty  \frac{\lambda^{k-2}}{k!} E(\Delta^k)
    & \;\le\;
    1 - \lambda E(\Delta) + \lambda^2 \sum_{k=2}^\infty  \frac{\lambda^{k-2}}{k!} E(|\Delta|^k).
  \end{align*}
	
	Since all terms of the last sum are positive, 
  we obtain for all $\gamma\ge
  \lambda$
  \begin{align*}
  E(e^{-\lambda \Delta})  & \;\le\;
    1 - \lambda E(\Delta) + \frac{\lambda^2}{\gamma^2} \sum_{k=2}^\infty \frac{\gamma^{k}}{k!} E(|\Delta|^k)\\
     & \;\le\;
		1 - \lambda E(\Delta) + \frac{\lambda^2}{\gamma^2} \sum_{k=0}^\infty \frac{\gamma^{k}}{k!} E(|\Delta|^k) 
		 \le 1 -\lambda \epsilon + \lambda^2  \underbrace{\frac{E(e^{\gamma \lvert\Delta\rvert})}{\gamma^2}}_{=:C(\gamma)},
  \end{align*}
	where the last inequality uses the first condition in the theorem.
 
  Given any $\gamma>0$, choosing $\lambda:=\min\{\gamma,
  \epsilon/(2C(\gamma))\}$ results in
  \[
  E(e^{-\lambda \Delta})  \;\le\; 1- \lambda\epsilon + \lambda\cdot \frac{\epsilon}{2C(\gamma)}\cdot C(\gamma) 
  \;=\; 1-\frac{\lambda \epsilon}{2} \;=\;  1-\frac{1}{p(\ell)}
  \]
  with $p(\ell):=2/(\lambda \epsilon) = \Theta(1/\lambda)$ since 
  only $\lambda$ but not~$\epsilon$ is allowed to depend on~$\ell$. 
	
	The aim is now to choose $\gamma$ in such a way that $E(e^{\gamma \lvert\Delta\rvert})=O(r(\ell))$. 
	Using Lemma~\ref{lem:general-bound-expec} with $f(x):=e^{\gamma x}$ and $\xmin:=0$, 
	we get
	\[
	E(e^{\gamma \lvert\Delta\rvert}) \;\le\; \sum_{j=0}^\infty e^{\gamma (j+1)} \Prob(|\Delta| \ge j) 
	\;\le\; 
	\sum_{j=0}^\infty e^{\gamma (j+1)} \frac{r(\ell)}{(1+\delta)^{j}},  
	\]
	where the last inequality holds by the second condition of the theorem.

  Choosing $\gamma:=\ln(1+\delta/2)$, which does not depend
  on $\ell$ since $\delta$ is a constant, yields
  \begin{align*}
    E(e^{\gamma \lvert\Delta\rvert})  
		& \;\le\; r(\ell) \cdot \sum_{j=0}^\infty \frac{(1+\delta/2)^{j+1}}{(1+\delta)^{j}} 
     \;=\; r(\ell)\cdot \left(1+\frac{\delta}{2}\right) \cdot \sum_{j=0}^\infty \left(1-\frac{\delta/2}{(1+\delta)}\right)^j \\
		& \;= \;
		r(\ell)\cdot \left(1+\frac{\delta}{2}\right) 
		\left(2+\frac{2}{\delta}\right) 
    \;\le\; r(\ell)\cdot (4+\delta+2/\delta).
  \end{align*}
  Hence $C(\gamma)\le
  r(\ell)(4+\delta+2/\delta)/\!\ln^2(1+\delta/2)$, which means
  $C(\gamma)=O(r(\ell))$ since $\delta$ is a constant.  By our choice
  of~$\lambda$, we have $\lambda\ge \epsilon/(2C(\gamma)) =
  \Omega(1/r(\ell))$ since also $\epsilon$ is a constant. Since
  $p(\lambda)=\Theta(1/\lambda)$, we know $p(\ell)=O(r(\ell))$, too.
  Condition~\eqref{eq:maindriftcondition} of
  Theorem~\ref{theo:orig-drift} has been established along with these
  bounds on~$p(\ell)$ and $\lambda=\lambda(\ell)$.

  To bound the probability of a success within $L(\ell)$ steps, we
  still need a bound on $D(\ell)=\max\{1,E(e^{-\lambda(X_{t+1}-b)}\mid
  X_t\ge b)\}$. If $1$ maximizes the expression then $D(\ell)\le r(\ell)$ follows. 
  Otherwise, 
  \begin{align*}
  & D(\ell)  \;=\; E(e^{-\lambda(X_{t+1}-b)}\mid X_t\ge b)
  \;\le\; E(e^{-\lambda(X_{t+1}-X_t)}\mid X_t\ge b) \\
  &\, \;\le\;
	E(e^{\lambda \lvert X_{t+1}-X_t\rvert}\mid X_t\ge b)
  \;\le \; E(e^{\gamma \lvert X_{t+1}-X_t\rvert}\mid X_t\ge b),
  \end{align*}
  where the first inequality follows from $X_t\ge b$ 
  and 
  the second one by $\gamma\ge \lambda$.  The last term can be bounded 
  as in the above calculation leading to $E(e^{\gamma \lvert\Delta\rvert})=O(r(\ell))$ since 
  that estimation uses only the second condition, which holds 
  conditional on $X_t> a$. Hence, in any case also $D(\ell) =
  O(r(\ell))$.  
  Altogether, we
  have 
  \[
  e^{-\lambda(\ell)\cdot \ell}\cdot D(\ell) \cdot p(\ell) \;=\; 
  e^{-\Omega(1/r(\ell))\cdot \ell} \cdot O((r(\ell))^2)
  \;=\; e^{-\Omega(\ell/r(\ell)) + O(\log(r(\ell)))}
  = 2^{-\Omega(\ell/r(\ell))},
  \]
  where the last simplification follows since
  $r(\ell)=o(\ell/\!\log(\ell))$ by prerequisite.  Choosing
  $L(\ell)=2^{c^*\ell/r(\ell)}$ for some sufficiently small
  constant~$c^*>0$, Theorem~\ref{theo:orig-drift} yields
  \[
  \Prob(T(\ell)\le L(\ell)) \;\le\; L(\ell)\cdot 2^{-\Omega(\ell/r(\ell))} \;=\;
  2^{-\Omega(\ell/r(\ell))},\]
  which proves the theorem.  \end{proof}

\section{Previous Applications}
\label{sec:applications}
Obviously the new condition does not impact RLS with constant size neighborhoods.
Hence the stronger condition immediately holds for \cite{SudholtAlgo2011} in the context of simulated
annealing.
We would expect that the condition is already met in all the proofs with standard bit mutation. 
In particular, when it is necessary to flip at least $j$ bits to decrease the distance from the target
of an amount $j$, generally the $j$ bit-flips are required also to increase the distance by $j$, implying 
that the proofs immediately carry over for the new stronger condition.
In this case, the condition holds for $\delta=1$ and $r(\ell)=2$:
\[
 P(\Delta_t(i) \leq +j) \leq {n \choose j} \bigg(\frac{1}{n}\bigg)^j\leq \frac{1}{j!} 
\leq 2 \cdot \bigg(\frac{1}{2}\bigg)^j
\]
This is the case for theorems 5, 6 and 7 of 
\cite{OlivetoWittPPSN2008, OlivetoWittAlgorithmica2011} respectively considering the 
\oneoneea for the Needle-in-a-haystack function and the \oneoneea with fitness proportional selection
for linear functions and for \textsc{OneMax}.
With at most small variations, such a case also applies to \cite{Monotone} relatively to the \oneoneea
for monotone functions, to \cite{Handl2008,LehreEcj2010} 
concerning multi-objective problems, to \cite{LehreSeal,LehreSoftcomputing} related to computing
unique input output sequences in software engineering, to \cite{lehreDyn} for dynamic optimisation
and to \cite{EDAs} concerning Estimation of Distribution Algorithms (EDAs).

Although the calculations do not imply only flipping $j$ bits to decrease the distance by $j$,
the bounds obtained in \cite{olivetoTEVC} (i.e. vertex cover instances) 
for decreasing the distance by $j$ also apply to increasing the distance by $j$.
The same holds for \cite{LehreWitt2010} where the calculations for the stronger condition remain 
the same as long as the Chv\'{a}tal bound in the opposite direction is applied 
(i.e. $Pr[X \leq E(X) - r \delta] \leq (\exp(-2\delta^2r)$) \cite{Chvatal}.

For some previous applications, a bound on the probability of performing jumps increasing
the distance from the optimum needs to be derived. 
This is the case for the maximum matching application (\ie, Theorem 8) 
of \cite{OlivetoWittPPSN2008, OlivetoWittAlgorithmica2011}, for Theorem 3 of \cite{DiversityEcj}
analysing a fitness-diversity mechanism, for Theorem~8 of \cite{NOWGECCO2009} analysing 
a fitness-proportional EA,  
and for Theorem 7 of \cite{SudholtRowe2012} analysing the (1,$\lambda$)-EA for OneMax. 
In the latter the Simplified Drift theorem considering self-loops, devised there, also 
needs to be adjusted to consider the stronger drift condition.
In the following we will discuss how to derive the missing bounds 
and show how the stronger condition holds with simple calculations.

\subsection{Maximum Matching (Theorem 8, \cite{OlivetoWittPPSN2008, OlivetoWittAlgorithmica2011})}
The probability of decreasing the augmenting path by $j$ was bounded by 
$p_{-j} \leq (j+1)/m^{2j}$ which was also used by Giel and Wegener in \cite{Giel2003}.
To prove the stronger condition we also need to bound the probability that the augmenting path is 
lengthened by $j$. Since there are at most $2h$ choices to lenghthen the path by $1$, 
there are at most $(2h)^j$ different ways to lengthen the path by $j$. 
Given that $2$ edges need to flip for each of the $j$ lengthenings, the probability of performing a jump
of at least length $j$ is bounded by 
\[
p(|\Delta_t| \geq j) \leq \sum_{i=j}^{m-j} \frac{(2h)^i}{m^{2i}} \leq
m \cdot \frac{(2h)^j}{m^{2j}} \leq \frac{m^{j+1}}{m^{2j}} = \frac{1}{m^{j-1}}  
\]
where the second inequality is obtained by considering that the term of the sum is maximised for $i=j$.
Then, by considering relevant steps 
(i.e. $p_{\text{rel}} \geq 1/(e m^2)$)
as in the rest of the proof, Condition 2, with $\delta=1$ and $r=22$ follows from:
 \begin{align*}
    \frac{p_{j+}}{p_{\text{rel}}} & 
    \;\le\; \mathord{\min}\mathord{\left\{1,\frac {em^2} {m^{j-1}} \right\}}
    \;\le\; \mathord{\min}\mathord{\bigg\{1, \frac {e} {m^{j-3}}\bigg\}} \;\le\; 22\cdot \left(\frac{1}{2}\right)^{j}
  \end{align*}
for $m\geq 2$.

\subsection{Diversity with fitness duplicates (Theorem 3, \cite{DiversityEcj})}
The probability of increasing the potential of the population $\varphi(P_t)$ by $j$ needs to be derived 
(i.e. $P(\Delta_\varphi = +j)$).
It turns out that most of the calculations for bounding the probability in the opposite direction apply.

In order to increase the potential by $j$, it is necessary to select some $y_k$ with $0\le k \le \varphi$
and flip $\varphi - k + j$ 1-bits into 0-bits. The probability is:
\begin{align*}
 P(\Delta_\varphi \leq +j) 
& \leq \frac{1}{\mu} \sum_{k=0}^{\varphi} \frac{1}{(\varphi-k+j)!} = 
\frac{1}{\mu} \sum_{k=0}^{\varphi} \frac{1}{(k+j)!}
\leq  \frac{1}{\mu} \sum_{k=0}^{\infty} \frac{1}{(k+j)!}
\end{align*}
and the rest of the proof follows exactly the same calculations as for the drift condition in the opposite direction.

\subsection{(1,$\lambda$)-EA Analysis (Theorem 7, \cite{SudholtRowe2012})}
The probability of performing jumps of length $j$ away from the optimum of Onemax is:
\[
p_{k,k+j} \leq \lambda {n \choose j} \frac{1}{n^j} \leq \frac{\lambda}{j!} \leq 
2 \lambda \bigg(\frac{1}{2}\bigg)^j 
\]
The probability of self-loops is bounded by:
\[
p_{k,k} = 1 -\bigg( 1- (1 - 1/n)^n\bigg)^\lambda \geq 1- (1-1/(2e))^\lambda \geq 1 - c^{\lambda}
\]
Hence, 
\[
\frac{p_{k,k+j}}{p_{k,k}} \leq \frac{2 \lambda }{1 -c^{\lambda}} \bigg(\frac{1}{2}\bigg)^j
\]
and the condition is established for any $\lambda= r(\ell)$ with $1\leq r(\ell) = o(\ell/\log(\ell))$.
Since in Theorem 7 \cite{SudholtRowe2012}, $\lambda \leq (1- \epsilon) \log_{\frac{e}{e-1}} n$, the Simplified Drift Theorem implies a runtime of 
at least $2^{\frac{cn^{\epsilon/2}}{(1-\epsilon)\log n}}$ with probability 
$1 - 2^{-\Omega\big(\frac{n^{\epsilon/2}}{\log n}\big)}$. 
This is a $\log n$ factor weaker than the statement of Theorem 7 \cite{SudholtRowe2012} where jumps towards
the target are not considered.  

\subsection{Fitness-Proportional EA (Theorem 8, \cite{NOWGECCO2009})}
The fitness-proportional EA (PEA) analysed on 
\textsc{OneMax} in \cite{NOWGECCO2009} 
works with a population of size~$\mu$, fitness-proportional 
selection, and mutation as only variation operator. It is proved for 
all polynomial population sizes that the algorithm is ineffective, 
using an appropriately defined potential 
function with drift away from the target and small 
probablity of jumping towards the target. 
However, we have noticed that 
very large jumps \emph{away} from the target 
are possible even for the smallest non-trivial population size of $\mu=2$. For instance, the analysis does not exclude situations where 
the two individuals $x_1$ and~$x_2$ differ strongly in their fitness, \eg, $\card{x_1}=0.999n$ and $\card{x_2}=0.5n$. The potential 
function $g(x_1,x_2)=8^{\card{x_1}}+8^{\card{x_2}}$ used in the paper scales fitness values exponentially. Suppose selection chooses 
$x_2$ for reproduction twice, which has probability $\Omega(1)$. Then the offspring population would consist of two 
individuals with a number of one-bits close to $n/2$. It is easy to see that this reduces the potential drastically, which corresponds 
to a step away from the target that is larger than allowed by the corrected drift theorem.

We can overcome this anomaly by replacing PEA with a modified algorithm PEA' with stronger selection 
pressure. Let $x_1,\dots,x_{\mu}$ be the individuals of the current population and assume 
\wlo\ that $f(x_1)\ge \dots \ge f(x_\mu)$. Let $\ftot=f(x_1)+\dots+f(x_\mu)$. 
A generation of PEA' does the following, where ``mutate'' means the usual 
standard bit mutation.
\begin{enumerate}
\item Mutate $x_1$ and add the result to the offspring population.
\item
For $t=2,\dots,\mu$:
\begin{enumerate} 
\item 
Choose a parent $x$, where 
\begin{itemize}
\item the probability of choosing $x_1$ is $\frac{f(x_1)}{\ftot} - \frac{1}{\mu}$ and
\item  
the probability of choosing $x_i$ is 
$\frac{f(x_i)}{\ftot} + \frac{1}{\mu\cdot (\mu-1)}$ for $2\le i\le \mu$.
\end{itemize}
\item Mutate~$x$ and add the result to the offspring population.
\end{enumerate}
\end{enumerate}

In other words, the best individual is selected and mutated at least once. This prevents the 
case outlined above where the potential of the offspring population drops drastically 
compared to the parent population. In the following steps, 
the other individuals are selected with higher probability than in the original algorithm. It is easy
to verify that PEA' always uses a well-defined probability distribution on the population when creating 
the offspring population and that its optimization time on \textsc{OneMax} is stochastically 
at most as large as the one of the original PEA.

The analysis in \cite{NOWGECCO2009} relies on the random $S_i$, where $S_i$ is the number 
of times $x_i$ is chosen for mutation. Since it is proved that $f(x_i)/\ftot \le 2/\mu$ for $1\le i\le \mu$, 
it follows immediately for the original PEA that $E(S_i)\le 2$, which is crucial 
for the rest of the analysis. We show that $E(S_i)\le 2$ also in the PEA'. 
First we get $E(S_1)=1+(\mu-1)f(x_1)/\ftot - (\mu-1)/\mu \le \mu f(x_1)/\ftot \le 2$, 
where the first inequality uses 
$f(x_1)/\ftot\ge 1/\mu$ and the second one uses $f(x_i)/\ftot \le 2/\mu$ for $i\ge 1$. 
For $i\ge 2$, we get 
$E(S_i) \le (\mu-1)f(x_i)/\ftot + 1/\mu \le (\mu-1) \cdot 2/\mu + 1/\mu \le 2$. Now
 the rest 
of the original analysis can be carried over.

The authors will use the arguments outlined above in an extended journal submission 
of their paper \cite{NOWGECCO2009}. This submission is currently under preparation.

\section*{Acknowledgements}
We would like to thank Per Kristian Lehre, Timo K{\"o}tzing and Christine Zarges for useful discussions.

\bibliographystyle{abbrv}
\bibliography{drift}

\end{document}